\documentclass[12pt]{article}
\title{A Framework for Fluid Motion Estimation using a Constraint-Based Refinement Approach}
\author{Hirak Doshi\footnote{Corresponding Author}\: and N. Uday Kiran\\Department of Mathematics and Computer Science\\Sri Sathya Sai Institute of Higher Learning, Andhra Pradesh, India\\Email: \{hirakdoshi,\:nudaykiran\}@sssihl.edu.in}
\date{}
\usepackage{mathtools}
\usepackage{amsmath,amsfonts,amssymb,amsthm}
\usepackage{relsize}
\usepackage{xcolor}
\usepackage{subfig}
\usepackage{float}
\usepackage{multirow}
\usepackage{adjustbox}
\usepackage{tikz}
\usepackage{algorithm}
\usepackage[noend]{algpseudocode}

\usepackage[a4paper,left=2cm,top=3cm,right=2cm,bottom=3cm]{geometry}
\usepackage[utf8]{inputenc}

\usepackage{array}
\newcolumntype{L}{>{\centering\arraybackslash}m{11cm}}

\usepackage[T1]{fontenc}
\usepackage{tikz}
\usepackage{setspace}
\usetikzlibrary{shapes.arrows}
\usetikzlibrary{arrows,decorations.pathreplacing,positioning}
\theoremstyle{definition} 

\newtheorem{lemma}{Lemma}

\newtheorem{theorem}{Theorem}

\newtheorem*{keywords}{\small Keywords}

\newtheorem*{class}{Mathematics Subject Classification 2020}

\usepackage{hyperref}
\hypersetup{
	colorlinks=true,
	linkcolor=blue,
	citecolor=red,      
	urlcolor=cyan,
}
\allowdisplaybreaks
\begin{document}
	\onehalfspacing
	\fontsize{13pt}{13pt}\selectfont
	\maketitle
	\begin{abstract}
		
		Physics-based optical flow models have been successful in capturing the deformities in fluid motion arising from digital imagery. However, a common theoretical framework analyzing several physics-based models is missing. In this regard, we formulate a general framework for fluid motion estimation using a constraint-based refinement approach. We demonstrate that for a particular choice of constraint, our results closely approximate the classical continuity equation-based method for fluid flow. This closeness is theoretically justified by augmented Lagrangian method in a novel way. The convergence of Uzawa iterates is shown using a modified bounded constraint algorithm. The mathematical well-posedness is studied in a Hilbert space setting. Further, we observe a surprising connection to the Cauchy-Riemann operator that diagonalizes the system leading to a diffusive phenomenon involving the divergence and the curl of the flow. Several numerical experiments are performed and the results are shown on different datasets. Additionally, we demonstrate that a flow-driven refinement process involving the curl of the flow outperforms the classical physics-based optical flow method without any additional assumptions on the image data.

	\end{abstract}
\begin{keywords}
	\small Fluid Motion Estimation, Evolutionary PDE, Image Processing, Augmented Lagrangian, Bounded Constraint Algorithm.
\end{keywords}
	\begin{class}
	35J47, 35J50, 35Q68.
	\end{class}
	\section{Introduction}
	\label{sec:1}

Variational models for motion estimation have always been one of the central topics in mathematical image processing. Since the seminal work of Horn and Schunck \cite{Horn} on the variational approach to optical flow motion estimation, many in-depth studies on this topic have been done by developing different variational models of optical flow to obtain useful insights into motion estimation (e.g. \cite{Schnorr,Aubert,Hinterberger,Weickert}). Many of these literature works on motion estimation have been focused on the constancy assumption e.g., the brightness constancy leading to an algebraic equation, in $(u,v)$ the motion components, called the optical flow constraint (OFC):
\begin{equation}
	f_\tau+f_{x}u+f_{y}v=0,
	\label{ofc}
\end{equation}
where $f(x,y,\tau)$ is the image sequence $f:\Omega\times [0,\infty)\to \mathbb{R}$ for an open bounded set $\Omega\subset\mathbb{R}^2$, $(x,y)$ are the spatial coordinates and $\tau$ is the time variable. However, constancy assumptions can't reflect the reality of actual motion because deformation effects of fluid, illumination variations, perspective changes, poor contrast etc. would directly affect the important motion parameters. For this reason, physics-dependent motion estimation algorithms have been widely investigated. 

A relatively recent work of Corpetti et al. \cite{Corpetti,corpetti,corpetti2} used the image-integrated version of the continuity equation from fluid dynamics. Since then, there has been a lot of attention on studying physics-based motion estimation. Estimates on the conservation of mass for a fluid in a digital image sequence were discussed by Wildes et.al. \cite{wildes}. The image intensity was obtained as an average of the object density, with the incident light parallel to the $z$-axis such that the 2D projection of the image intensity can be captured as
\begin{equation*}
	f(x,y,\tau)=\int_{z_1}^{z_2} \rho(x,y,z,\tau) \:dz.
\end{equation*}
Further using fluid mechanics models, Liu et al. \cite{Liu} provided a rigorous framework for fluid flow by deriving the projected motion equations. Here the optical flow is proportional to the path-averaged velocity of fluid or particles weighted with a relevant field quantity. 
Luttman et al. \cite{Luttman} computed the potential (resp. stream) function directly by assuming that the flow estimate is the gradient of a potential function (resp. symplectic gradient of a stream function). To overcome the limitations of the global smoothness regularization for fluid motion estimation, Corpetti et.al \cite{corpetti} proposed a second-order div-curl regularization for a better understanding of intrinsic flow structures. A detailed account of various works in fluid flow based on physics is given in \cite{heitz}. Here the authors observed a physical meaning associated with the terms of the continuity equation (CEC):
\begin{equation}
	\underbrace{f_\tau+f_{x}u+f_{y}v}_{\textbf{(a)}}+\underbrace{f(u_{x}+v_{y})}_{\textbf{(b)}}=0;
	\label{coneq}
\end{equation}
$\textbf{(a)}$ corresponds to the brightness constancy given in (\ref{ofc}) and $\textbf{(b)}$ is the non-conservation term due to loss of particles. Furthermore, a divergence-free approximation of the equation (\ref{coneq}) can be obtained by setting the term $\textbf{(b)}$ to zero. It is thus natural to study the effect of the term (\textbf{b}) in extracting the inherent fluid properties of the flow.

\subsection{Our Contribution}
In the current work, we develop a generic physics-based framework for fluid motion estimation for capturing intrinsic spatial characteristics and vorticities. It consists of a two-step technique to compute fluid motion from a sequence of scalar fields or illuminated particles transported by the flow. The technique proposed consists in filtering with an appropriate semi-group the divergence of an initial velocity field estimated through a classical Horn and Schunck (HS) estimator. The vorticity of the initial solution is kept constant while the flow is transformed through its divergence in order to obey a continuity equation for the brightness data which has been used in several optical flow estimators dedicated to fluid flow.

In particular, our method uses a constraint-based refinement approach. As mentioned above, in the two-step technique, the first requirement is an initial flow estimate $(u_0,v_0)$ that obeys the classical optical flow principles like brightness constancy and pixel correspondence. This estimate may not be able to capture the underlying geometric features of the fluid flow. The main idea is to perform a refinement over this crude estimate to capture precise flow structures driven by additional constraints specific to applications. As a concrete example, we choose the initial estimate coming from the Horn and Schunck model \cite{Horn}. It is well-known that this model is not well-suited for fluid motion estimation. In fact the global smoothness regularization damps both the divergence and vorticity of the motion field. We show in particular how this model can be adapted and refined through our approach. A special feature of our model is the diagonalization by the Cauchy-Riemann operator leading to a diffusion on the curl and a multiplicative perturbation of the laplacian on the divergence of the flow.

There are two main advantages of our method. Firstly, from a theoretical perspective it provides us with an evolutionary PDE setup which allows a rigorous mathematical framework for the well-posedness discussion. Secondly, the contractive semigroup structure on the divergence leads to a simpler numerical analysis. A modified bounded constraint algorithm \cite{nocedal} is employed to theoretically show the convergence of the dual variable introduced by the augmented Lagrangian formulation. The inner iterations of the algorithm use the contractive semigroup of the elliptic term. This approach thus allows us to build a quantitative connection between the optical flow and the fluid flow for various flow visualizations which is often a key problem.

The paper is organized as follows. In Section \ref{sec:2}, we give a detailed description of our model. Section \ref{sec:3} is devoted to the mathematical framework. Here we discuss the mathematical well-posedness and the regularity of the solutions. We also show the diagonalization process under the application of the Cauchy-Riemann operator. In Section \ref{sec:7}, we show how for a particular choice of additional constraints, our model closely approximates the continuity equation model using a modified augmented Lagrangian formulation. We also employ the bounded constraint algorithm to show the convergence of the Uzawa iterates. Finally, in Section \ref{sec:9}, we show our results on different datasets.

\section{Description of the Model}
\label{sec:2}

Our general formulation is given as
\begin{equation}
	J_{\text{R}}(\textbf{u})=\beta\int_\Omega \phi(f)\psi(\nabla\textbf{u})+\alpha\int_\Omega\{|\nabla u|^2+|\nabla v|^2\},
	\label{fun}
\end{equation}
where the constants $\alpha, \beta$ are weight parameters, the function $\psi$ depends on the components of the flow and its derivatives which essentially captures the underlying geometric structures and the function $\phi$ corresponds to an image-dependent weight term. A few possible combinations are summarized in the following table:

\begin{table}[H]
	\centering
	\begin{tabular}{|c|c|c|}
		\hline
		$\phi(f)$ & $\psi(\nabla\textbf{u})$ & Nature of the model\\ 
		\hline & & \\[-1.5ex]
		$f^2$ & $(\nabla\cdot\textbf{u})^2$ & Anisotropic, image-driven, penalizing divergence of the flow\\[0.3cm]
		1 & $(\nabla\cdot\textbf{u})^2$ & Isotropic, flow-driven, penalizing divergence of the flow\\[0.3cm]
		$f^2$ & $(\nabla_H\cdot\textbf{u})^2$ & Anisotropic, image-driven, penalizing curl of the flow\\[0.3cm]
		1 & $(\nabla_H\cdot\textbf{u})^2$ & Isotropic, flow-driven, penalizing curl of the flow\\
		\hline
	\end{tabular}
	\caption{Some choices for the functions $\phi$ and $\psi$}
	\label{tab}
\end{table}
As seen from the table, the function $\phi$ dictates whether the refinement process is image-driven or flow-driven. When $\phi(f)=1$, there is no influence of the image data on the additional constraint. As a result, the refinement process is completely flow-driven. We assume that both the functions $\phi$ and $\psi$ are real-valued smooth functions. Moreover, we assume $\phi(f)$ to be a monotone-increasing function. The first term of $J_{\text{R}}(\textbf{u})$ captures the non-conservation term that violates the constancy assumptions and the second term is the $L^{2}$ regularization which governs the diffusion phenomena. In the current work, we particularly focus on the case where $\psi=(\nabla\cdot\textbf{u})^2$, i.e. where the function $\psi$ penalizes the divergence of the flow. In this case, the refinement functional becomes
\begin{equation}
	J_{\text{R}}(\textbf{u})=\beta\int_\Omega \phi(f)(\nabla\cdot\textbf{u})^2+\alpha\int_\Omega\{|\nabla u|^2+|\nabla v|^2\}.
	\label{fun1}
\end{equation}
\subsection{Additional Constraint penalizing the Curl of the Flow}
In Table (\ref{tab}) we have suggested two such choices for $\psi$, one penalizing the divergence of the flow and the other penalizing the curl. The operator $\nabla_H:=(\partial_y,-\partial_x)$ is called the orthogonal gradient, also referred to as the symplectic gradient in the literature \cite{Luttman}. This geometric constraint captures the rotational aspects of the flow better. In this work, we will particularly demonstrate that a flow-driven refinement process involving the curl of the flow outperforms the classical physics-based optical flow method without any additional assumptions on the image data.

\section{Mathematical Framework}
\label{sec:3}
\subsection{Diagonalization of the System}
\label{sec:4}

The associated system of PDEs for the variational formulation is given as:

\begin{equation}
	\begin{cases}
		\dfrac{\partial u}{\partial t}=\Delta u + a_0 \dfrac{\partial}{\partial x}[\phi (f)(u_x+v_y)] \text{ in }\Omega\times (0,\infty),\\[0.5cm]
		\dfrac{\partial v}{\partial t}=\Delta v + a_0 \dfrac{\partial}{\partial y}[\phi (f)(u_x+v_y)] \text{ in }\Omega\times (0,\infty),\\[0.5cm]
		u(x,y,0) = u_0 \text{ in } \Omega,\\[0.3cm]
		v(x,y,0) = v_0 \text{ in } \Omega,\\[0.3cm]
		u = 0 \text{ on }\partial\Omega\times (0,\infty),\\[0.3cm]
		v = 0 \text{ on }\partial\Omega\times (0,\infty).
	\end{cases}
	\label{sys}
\end{equation}
Here $(u_0,v_0)$ is the initial flow estimate obtained from the pixel correspondence, $a_0=\beta/\alpha$ is a positive constant. Since there is no pixel motion at the boundary, it is natural to work with Dirichlet boundary conditions. We will show that the system (\ref{sys}) can be diagonalized by an application of Cauchy-Riemann operator. This special feature is intriguing as well as of great advantage for later analysis. Let us first rewrite the system (\ref{sys}) as
\begin{equation}
	\frac{\partial \textbf{u}}{\partial t}=A\textbf{u},
	\label{pl3} 
\end{equation}
where
\begin{equation*}
	\frac{\partial \textbf{u}}{\partial t}=
	\begin{bmatrix}
		\frac{\partial u}{\partial t}\\[0.3cm]
		\frac{\partial v}{\partial t}
	\end{bmatrix},
	\qquad
	A\textbf{u}=
	\begin{bmatrix}
		\Delta u + a_0 \frac{\partial}{\partial x}[\phi (f)(u_x+v_y)]\\[0.3cm]
		\Delta v + a_0 \frac{\partial}{\partial y}[\phi (f)(u_x+v_y)]
	\end{bmatrix}.
\end{equation*}
As we are in the Sobolev setting, the derivatives are taken in a distributional sense. Thus, a key observation is that the order of derivatives can be interchanged. Let us denote by $R$ the Cauchy-Riemann operator matrix
\begin{equation*}
	R=\begin{bmatrix}
		\partial_y & -\partial_x \\[0.3cm]
		\partial_x & \partial_y	
	\end{bmatrix}.
\end{equation*}
Acting $R$ on both sides of (\ref{pl3}) leads to
\begin{equation*}
	R\Big(\frac{\partial \textbf{u}}{\partial t}\Big)=RA\textbf{u}.
\end{equation*}
This leads to the following transformation of the original coupled system
\begin{align*}
	\begin{bmatrix}
		\partial_y & -\partial_x \\[0.3cm]
		\partial_x & \partial_y	
	\end{bmatrix}
	\begin{bmatrix}
		\frac{\partial u}{\partial t}\\[0.3cm]
		\frac{\partial v}{\partial t}
	\end{bmatrix}
	&=
	\begin{bmatrix}
		\partial_y & -\partial_x \\[0.3cm]
		\partial_x & \partial_y	
	\end{bmatrix}
	\begin{bmatrix}
		\Delta u + a_0 \frac{\partial}{\partial x}[\phi (f)(u_x+v_y)]\\[0.3cm]
		\Delta v + a_0 \frac{\partial}{\partial y}[\phi (f)(u_x+v_y)]
	\end{bmatrix}\\[0.3cm]
	&=
	\begin{bmatrix}
		\Delta & 0\\[0.3cm]
		0 & \Delta\circ k 
	\end{bmatrix}
	\begin{bmatrix}
		\partial_y & -\partial_x \\[0.3cm]
		\partial_x & \partial_y
	\end{bmatrix}
	\begin{bmatrix}
		u \\[0.3cm]
		v
	\end{bmatrix},
\end{align*}
where, with a slight abuse of notation we denote $k$ for the function $1+a_0\phi(f)$ and for the multiplicative operator $x\mapsto kx$. Since $\phi$ is a  bounded function and $a_0>0$, the multiplicative term $k$ is bounded and strictly positive. We have thus obtained the following decoupling:
\begin{equation}
	\frac{\partial }{\partial t}(R\textbf{u})=DR\textbf{u},
	\label{pl4} 
\end{equation}
where 
\begin{equation*}
	D=
	\begin{bmatrix}
		\Delta & 0\\[0.3cm]
		0 & \Delta\circ k 
	\end{bmatrix},
	\qquad
	R=
	\begin{bmatrix}
		\partial_y & -\partial_x \\[0.3cm]
		\partial_x & \partial_y
	\end{bmatrix}.
\end{equation*}
The application of the Cauchy-Riemann operator $R$ on the system has resulted in the following relation
\begin{equation*}
	D= RAR^{-1}.
\end{equation*}
The operator $A$ has been diagonalized by the matrix $R$. The decoupled system (\ref{pl4}) thus takes the following form:
\begin{align}
	\begin{cases}
		\dfrac{\partial \xi}{\partial t}&=\Delta \xi,\\[0.5cm]
		\dfrac{\partial \zeta}{\partial t}&=\Delta (k\zeta),
	\end{cases}
	\label{decsys}
\end{align}
where $\xi:=u_{y}-v_{x}$ is the curl and $\zeta:=u_{x}+v_{y}$ is the divergence of the flow and $k=1+a_{0}\phi(f)$ is a multiplicative factor.

\subsection{Multiplicative Perturbation of the Laplacian}
\label{sec:5}
Let us consider the divergence equation from (\ref{decsys})
\begin{equation}
	\frac{\partial \zeta}{\partial t}=\Delta (k\zeta),
	\label{per}
\end{equation}
where $k=1+a_0\phi(f)$. We make a change of variable $\eta=k\zeta$. This transformation leads to the equation:

\begin{equation}
	\frac{\partial \eta}{\partial t}=k\Delta \eta.
	\label{pert}
\end{equation}
The operator $k\Delta$ is the multiplicative perturbation of the laplacian. It arises in many physical phenomena e.g. in the theory of wave propagation in non-homogeneous media. The operator has been studied in appropriate weighted function spaces see \cite{eidus,altomare}. In our case the multiplicative perturbation $k\Delta$ leads to an image driven perturbation because of the $k$ factor which depends on the image $f$. The authors in \cite{altomare} derive an approprixation for the kernel of the associated semigroups by positive linear integral operators
\begin{equation}\label{anisotropic_gaussian}
	G_m(f)(x)=\Big(\frac{m}{4\pi k(x)}\Big)^{n/2}\int_{\mathbb{R}^n}f(y)\exp\Big(-\frac{m}{4k (x)}|x-y|^2\Big)\:dy.
\end{equation}
Using the above kernel one can design an appropriate stencil for convolution. It is also interesting to note that the stencil size varies with respect to the image intensity. This aspect we will discuss in a forthcoming paper. Let us now consider the Gaussian kernel associated with the operator $k\Delta$
\begin{equation*}
	G_k(x,t):=\frac{1}{4\pi k(x)t}\exp\Big(-\frac{|x|^2}{4k(x)t}\Big).
\end{equation*} 
The perturbation $k$ plays an important role in controlling the rate of diffusion. If $k$ is large then the Gaussian becomes broader and shorter and if it is small then the Gaussian is thinner and taller. Since $\Omega$ is bounded, the perturbation $k$ is bounded. Hence there exists $a_1,a_2>0$ such that $a_1\le k(x)\le a_2$. Using this fact, we can obtain the following bound on $G_k(x,t)$.
\begin{lemma}
	Let $n=2$. Then
	\begin{equation*}
		\|G_k(x,t)\|_p\le C_k t^{(1-p)/p},
	\end{equation*}
	where the constant
	\begin{equation*}
		C_k=\frac{1}{4\pi a_1}\Big(\frac{4\pi a_2}{p}\Big)^{1/p}.
	\end{equation*}
\end{lemma}

\subsection{Wellposedness and Regularity}
\label{sec:6}
	Let us now consider the abstract IVP associated with the first equation in (\ref{decsys}):
\begin{equation}
	\begin{cases}
		\dfrac{d\xi}{dt} + A_1\xi = 0 \text{ on } [0,\infty),\\[0.5cm]
		\xi(0)=\xi_0 \in L^2(\Omega).
		\label{eivp}
	\end{cases}
\end{equation}
where the initial data
\[
\xi_0 = \partial_yu_0-\partial_xv_0,
\]
and $(u_0,v_0)$ is the Horn and Schunck optical flow. Here $A_1:D(A_1)\to \mathcal{H}_1$ is an (unbounded) operator
\begin{equation*}
	\begin{cases}
		D(A_1)=\{\xi\in H^2(\Omega)\cap H^1_0(\Omega):A_1\xi\in L^2(\Omega)\},\\[0.3cm]
		\mathcal{H}_1=L^2(\Omega),\\[0.3cm]
		A_1\xi:=-\Delta \xi.
	\end{cases}
\end{equation*}
The operator $A_1$ is maximal monotone and symmetric. Hence it is self adjoint. For the well-posedness of the problem (\ref{eivp}) we refer to (Theorem $7.7$ and $10.1$ in \cite{brezis}). Similarly, the abstract IVP for the second equation becomes 
\begin{equation}
	\begin{cases}
		\dfrac{d\eta}{dt} + A_2\eta = 0 \text{ on } [0,\infty),\\[0.5cm]
		\eta(0)=\eta_0 \in L^2(\Omega).
		\label{eivp2}
	\end{cases}
\end{equation}
where the initial data
\begin{equation*}
	\eta_0=k(\partial_x u_0+\partial_y v_0),
\end{equation*}
is the weighted divergence of the Horn and Schunck optical flow. Here  $A_2:D(A_2)\to \mathcal{H}_2$ is the (unbounded) operator
\begin{equation*}
	\begin{cases}
		D(A_2)=\{\eta\in H^2(\Omega)\cap (H^1_0)_k(\Omega): A_2\eta\in L^2_k(\Omega)\},\\[0.3cm]
		\mathcal{H}_2=L^2_k(\Omega),\\[0.3cm]
		A_2 \eta=-k\Delta \eta.
	\end{cases}
\end{equation*}
Also $\eta_0\in L^2(\Omega)$. Here the operator $k\Delta$ is a multiplicative perturbation of the laplacian where $k$ is bounded and strictly positive since $\phi$ is a monotone increasing function. The problem will be studied in weighted Sobolev space $\mathcal{H}_2$ following a similar approach as in \cite{eidus}. The space $\mathcal{H}_2$ is a Hilbert space equipped with the inner product
\begin{equation*}
	\langle w_1,w_2\rangle_{\mathcal{H}_2}:=\int_\Omega \frac{1}{k} \:w_1w_2\:dx
\end{equation*}
and the norm
\begin{equation*}
	\|w\|^2_{\mathcal{H}_2}=\int_\Omega \frac{1}{k} \:|w|^2\: dx.
\end{equation*}
Similarly the Hilbert space $\mathcal{H}_3:=(H^1_0)_k(\Omega)$ has the norm \cite{eidus}
\begin{equation*}
	\|w\|^2_{\mathcal{H}_3}:=\|w\|^2_{\mathcal{H}_2}+\|\nabla w\|_{\mathcal{H}_1}^2
	= \int_\Omega \Big(\frac{1}{k}|w|^2+|\nabla w|^2\Big).
\end{equation*}
The operator $A_2$ is symmetric and maximal monotone. It is also interesting to note that in our context the weight term $k$ in $\mathcal{H}_2$ is actually dependent on the image $f$ - bringing in an anisotropy into the discussion. Thus $\mathcal{H}_2$ is an image dependent Sobolev space. When $\phi(f)$ is a constant function, i.e. the case where the refinement is independent of the image, the norms $||\cdot||_{\mathcal{H}_1}$ and $||\cdot||_{\mathcal{H}_2}$ coincide upto a constant. In our context, as the values of the image are bounded, $k$ is a bounded function. Also in the pre-processing stage since the images are smoothened with a Gaussian filter we can further assume that $k$ is smooth. We will prove a result on the regularity of the solution for any non-zero time in the diffusion process.
\begin{theorem}
	Let $\eta_0\in \mathcal{H}_1$. Then the solution of the problem (\ref{eivp2}) satisfies
	\begin{align*}
		\eta\in C^1((0,\infty),\mathcal{H}_2)\cap C([0,\infty),H^2(\Omega)\cap \mathcal{H}_3).
	\end{align*}
	For any $\epsilon >0$ we also have
	\begin{equation}
		\eta\in C^\infty([\varepsilon,\infty)\times \overline{\Omega}).
		\label{rel1}
	\end{equation}
	Moreover, $\eta\in L^2((0,\infty),\mathcal{H}_3)$, and
	\begin{equation}
		\frac{1}{2}\|\eta(T)\|^2_{\mathcal{H}_1}+\int_0^T \|\nabla \eta(t)\|^2_{\mathcal{H}_1}\:dt =\frac{1}{2}\|\eta_0\|^2_{\mathcal{H}_2}
		\label{rel2}
	\end{equation}
	holds for $T>0$.
\end{theorem} 
\begin{proof}
	We only show the energy estimates as the remaining part of the proof follows very closely to the proof of theorem 10.1 in \cite{brezis}. It is clear that the operator $A_2$ is symmetric and maximal monotone. Hence it is self-adjoint. Therefore, by Theorem $7.7$  in \cite{brezis}, we have
	\begin{align*}
		\eta\in C^1((0,\infty),\mathcal{H}_2)\cap C([0,\infty),H^2(\Omega)\cap \mathcal{H}_3).
	\end{align*}
	Define $\sigma(t)=\frac{1}{2}\|\eta(t)\|^2_{\mathcal{H}_2}$. Since $\eta\in C^1((0,\infty),\mathcal{H}_2)$ it is clear that $\sigma$ is $C^1$ on $(0,\infty)$. Therefore
	\begin{align*}
		\sigma'(t)&=\Big\langle \eta(t),\frac{d\eta}{dt}(t)\Big\rangle_{\mathcal{H}_2}\\
		&=\Big\langle \eta(t),k\Delta \eta\Big\rangle_{\mathcal{H}_2}\\
		&=\Big\langle \eta(t),\Delta \eta\Big\rangle_{\mathcal{H}_1}\\
		&=-\|\nabla \eta(t)\|^2_{\mathcal{H}_1}.
	\end{align*}
	Integrating from $\varepsilon$ to $T$ where $0<\varepsilon<T<\infty$ we get
	\begin{equation*}
		\sigma(T)-\sigma(\varepsilon)=-\int_\varepsilon^T \|\nabla \eta(t)\|_{\mathcal{H}_1}^2\: dt.
	\end{equation*} 
	Again as $\eta\in C((0,\infty),\mathcal{H}_2(\Omega))$ we have $\sigma(\varepsilon)\to\sigma(0)=\frac{1}{2}\|\eta_0\|^2_{\mathcal{H}_2}$ as $\varepsilon\to 0$. Therefore in the limiting case we obtain
	\begin{equation*}
		\sigma(T)+\int_0^T \|\nabla \eta(t)\|_{\mathcal{H}_1}^2\:dt=\frac{1}{2}\|\eta_0\|_{\mathcal{H}_2}^2,
	\end{equation*}
	and (\ref{rel2}) holds. Integrating the Hilbert space $(H^1_0)_k(\Omega)$ norm defined above from $0$ to $T$ we get
	\begin{align*}
		\int_0^T\|\eta(t)\|^2_{(H^1_0)k}\:dt&=\int_0^T\|\eta(t)\|^2_{\mathcal{H}_2}\:dt+\int_0^T\|\nabla \eta(t)\|_{\mathcal{H}_1}^2\:dt\\
		&=2\int_0^T \sigma(t)\:dt+\int_0^T\|\nabla \eta(t)\|_{\mathcal{H}_1}^2\:dt\\
		&\leq 2\int_0^T \sigma(t)\:dt+ \frac{1}{2}\|\eta_0\|_{\mathcal{H}_2}^2.
	\end{align*}
	This shows that
	\begin{equation*}
		\eta\in L^2((0,\infty),\mathcal{H}_3).\qedhere
	\end{equation*}
\end{proof}

\section{A Special Case: Approximating the Continuity Equation Model}
\label{sec:7}
In this section we show that for a specific choice of the additional constraint $\phi(f)=f^2$, $\psi = (\nabla\cdot\textbf{u})^2$, our model closely approximates the CEC based model. We justify this theoretically using the modified augmented Lagrangian framework.

\subsection{The Augmented Lagrangian Framework}
Let $\mathcal{V}=H^1(\Omega)\times H^1(\Omega)$ and $\mathcal{H}=L^2(\Omega)$ denote the Hilbert spaces with the respective norms $\|\cdot\|_{\mathcal{V}},\|\cdot\|_{\mathcal{H}}$. For simplicity, we fix $\alpha=\beta=1$. Recall that our refinement model evolves over an initial estimate for which the pixel correspondence problem is already solved upto a certain level. This means that the optical flow constraint (OFC) is satisfied by the flow estimate. Taking this into account we recast the variational problem to a constrained minimization problem:

\begin{equation}
	\min_{\textbf{u}} J_{\text{R}}(\textbf{u})=\int_\Omega (f\nabla\cdot\textbf{u})^2+\int_\Omega (|\nabla u|^2+|\nabla v|^2),
	\label{varp}
\end{equation}
subject to the constraint 
\begin{equation}
	B\textbf{u}:=\nabla f\cdot\textbf{u}=-f_t=:c.
	\label{cons}
\end{equation}
where $\mu_R >0$ and $\lambda_1\in \mathcal{H}$ is the Lagrange multiplier. The associated augmented Lagrangian for the problem (\ref{varp})-(\ref{cons}) is:

\begin{equation}
	\mathcal{L}_{\mu_R}(\textbf{u},\lambda_1)=J_R(\textbf{u})+\frac{\mu_R}{2}\|B\textbf{u}-c\|^2+\langle\lambda_1,B\textbf{u}-c\rangle.
	\label{auglag7}
\end{equation}
Since OFC is a divergence-free approximation of the continuity equation data term, a similar constrained minimization problem can be considered:

\begin{equation}
	\min_{\textbf{u}} J_{\text{C}}(\textbf{u})=\int_\Omega (f_t+\nabla\cdot(f\textbf{u}))^2+\int_\Omega (|\nabla u|^2+|\nabla v|^2),
	\label{varp1}
\end{equation}
subject to the constraint 
\begin{equation}
	B\textbf{u}=c.
	\label{cons1}
\end{equation}
The associated augmented Lagrangian for the problem (\ref{varp1})-(\ref{cons1}) is:

\begin{equation}
	\mathcal{L}_{\mu_C}(\textbf{u},\lambda_1)=J_C(\textbf{u})+\frac{\mu_C}{2}\|B\textbf{u}-c\|^2+\langle\lambda_1,B\textbf{u}-c\rangle.
	\label{auglag3}
\end{equation}
We observe that
\begin{equation*}
	J_{\text{C}}(\textbf{u})=J_{\text{HS}}(\textbf{u})+J_{\text{R}}(\textbf{u})+K(\textbf{u}),
\end{equation*}
where $J_{\text{HS}}$ denotes the Horn and Schunck functional \cite{Horn} and
\begin{equation*}
	K(\textbf{u})=2\int_{\Omega}(B \textbf{u}-c)(f\nabla\cdot \textbf{u}).
\end{equation*}
By the Cauchy-Schwarz inequality we have $|K(\textbf{u})|^{2}\leq 2\|B\textbf{u}-c\|^2_{L^2}\|f\nabla\cdot \textbf{u}\|^{2}_{L^2}$. Thus, 
\begin{equation}
	J_{\text{C}}(\textbf{u})=J_{\text{R}}(\textbf{u})+O(\sqrt{\epsilon})\quad \text{ whenever }\quad \|B\textbf{u}-c\|_{L^2}=\textnormal{O}(\epsilon).
	\label{ceapprox}
\end{equation}
Heuristically arguing, we can take motivation from (\ref{ceapprox}) and adopt a two-phase strategy where we first determine the minimizer of $J_{\text{HS}}$ and use this minimizer as an initial condition in the evolutionary PDE associated with (\ref{fun}). Although our model is not derived by rigorous fluid mechanics, we still demonstrate that our results closely approximate the physics-based models for this particular choice of the functions $\phi$ and $\psi$. The first step is to show the equivalence of the variational problem with the associated saddle point problem, see \cite[Chapter 3]{glowinski} for further discussions.
\begin{lemma}
	$(\textbf{u},\lambda)$ is a saddle point of (\ref{auglag3}) iff $\textbf{u}$ solves the variational problem (\ref{varp1})-(\ref{cons1}).
	\label{1}
\end{lemma}
Observe that the augmented Lagrangian (\ref{auglag3}) can be reformulated as
\begin{equation}
	\mathcal{L}_{\mu_C}(\textbf{u},\lambda_1)=J_R(\textbf{u})+\frac{\mu_C}{2}\|B\textbf{u}-c\|^2+\langle\lambda_1+2f\nabla\cdot\textbf{u},B\textbf{u}-c\rangle.
	\label{auglag4}
\end{equation}
The parameters $\mu_C, \mu_R$ can be chosen as large as necessary. The lagrange multiplier $\lambda_1$ which acts as a dual variable is obtained by the Uzawa iteration
\begin{equation}
	\lambda_1^{(n+1)}=\lambda_1^{(n)}+2fd^{(n)}+\rho^{(n)}(B\textbf{u}^{(n)}-c),
	\label{uzawa}
\end{equation}
where $d^{(n)}=\nabla\cdot \textbf{u}^{(n)}$ and $\rho^{(n)}$ is a tuning parameter. To show the equivalence of the two unconstrained optimization problems it is necessary to show that the Lagrange multipliers \{$\lambda_1^{(n)}$\} converge. For this, we rely upon the techniques of bounded constraint algorithm, see \cite[Chapter 17]{nocedal} for more details.

\subsection{The Bounded Constraint Algorithm}
The starting point is the crude pixel correspondence obtained from the Horn and Schunck optical flow $\textbf{u}^{(0)}$ obtained within a tolerance limit $\delta_{\text{HS}}$ prescribed by Liu-Shen \cite{Liu3}. Observe that when the optical flow constraint is exactly satisfied then the formulations (\ref{varp}) and (\ref{varp1}) coincide. However in reality due to numerical errors and approximations, the constraints are never exactly met. Thus it is natural to look at the equality constraint $B\textbf{u}=c$ as a bounded constraint $\|B\textbf{u}-c\|_{\mathcal{H}}\le \epsilon_1^{(n)}$ where $\epsilon_1^{(n)}$ is a threshold parameter.
\begin{algorithm}[H]
	\caption{Bounded Constraint Algorithm}\label{BCA}
	\begin{algorithmic}[1]
		\State Set $\lambda^{(0)},\rho^{(0)}$. Choose $\epsilon_1^{(0)},\epsilon_2^{(0)}$.
		\State Obtain initial HS optical flow $\textbf{u}^{(0)}$
		\For{$n = 1,2,\dots$ until convergence \textbf{do}}
		\State update $\textbf{u}^{(n)}, d^{(n)}$
		\If $\|B\textbf{u}^{(n)}-c\|_{\mathcal{H}}\le \max\{\epsilon_1^{(n)},2\delta_{\text{HS}}\}$ 
		\If $\|fd^{(n)}\|_{\mathcal{H}}\le  \epsilon_2^{(n)}$
		\State break;
		\Else
		\State update $\lambda_1^{(n)}$ by (\ref{uzawa})
		\State $\rho^{(n+1)}\gets \rho^{(n)}$
		\State tighten tolerances $\epsilon_1^{(n+1)},\epsilon_2^{(n+1)}$
		\EndIf
		\Else 
		\State update $B\textbf{u}^{(n)}-c$
		\State $\lambda^{(n+1)}\gets\lambda^{(n)}$
		\State $\rho^{(n+1)}\gets 100\rho^{(n)}$
		\State tighten tolerances $\epsilon_1^{(n+1)},\epsilon_2^{(n+1)}$
		\EndIf
		\EndFor
	\end{algorithmic}
\end{algorithm}
The algorithm can be viewed in two phases. In the first phase, it is purely a diffusion process while in the second phase, the bounded constraint approximates the continuity equation constraint.This happens because a part of the OFC is already embedded in the CEC. The relaxation $\delta_{\text{HS}}$ is allowed so that we do not move too far away from the constraint and to ensure that the tolerance $\epsilon_1^{(n)}$ doesn't become too small.

\subsection{Description of the Bounded Constraint Algorithm}

The initial Horn and Schunck (HS) estimate $\textbf{u}^{(0)}$ first obtained. The updates $\textbf{u}^{(n)}$ at step 4 are obtained by discretizing the Euler-Lagrange equations:
\begin{equation}
	\textbf{P}\textbf{u}^{(n+1)}=\textbf{b}\textbf{u}^{(n)},
	\label{discrete}
\end{equation}
where
\begin{equation}
	\begin{split}
		\textbf{P}:= \begin{bmatrix}
			\alpha+\frac{2\beta f^2}{{\Delta x}^2} & 0 \\[0.3cm]
			0 & \alpha+\frac{2\beta f^2}{{\Delta y}^2}
		\end{bmatrix}, \quad
		\textbf{b}\textbf{u}^{(n)}=\begin{bmatrix}
			\alpha (M*u^{(n)})+\beta\frac{\partial}{\partial x}[f^2d^{(n)}]\\[0.3cm]
			\alpha (M*v^{(n)})+\beta\frac{\partial}{\partial y}[f^2d^{(n)}]
		\end{bmatrix},
	\end{split}
\end{equation}
$\Delta x,\Delta y$ are grid step sizes and $M$ is the nine-point approximation stencil of the laplacian. $d^{(n)}$ is updated by the relation \[d^{(n)}=S(t)d^{(n-1)},\] where the map $S(t): u_0\mapsto u(t):= G_k(\cdot,t)*u_0$ form a linear, continuous semigroup of contractions in $\mathcal{H}$, $G_k$ is the diffusion kernel associated with the operator $k\Delta$. This semigroup structure allows the process to preserve the spatial characteristics of the divergence and the vorticities.

The convergence criteria mentioned in steps 5 and 6 are checked subsequently. If both criteria are satisfied then the required approximation is met and the algorithm terminates. If step 6 is not satisfied, then we enter the inner iterations and run through steps 9 to 11, where the Lagrange multiplier $\lambda_1$ is updated while the penalty parameter $\mu_C$ is not modified. 

The outer iteration steps 13 to 16 are executed when the convergence criteria in step 5 fail. In this case, the optical flow constraint is first updated. Then, to preserve the convergence and to ensure no occurrence of spurious updates happens, the Lagrange multiplier is assigned the value from the previous iteration. The tuning parameter is updated with a higher penalty to ensure that the iterates remain within bounds.

\subsection{Convergence of the Uzawa Iterates}
So far we have discussed how the modified augmented Lagrangian formulation is employed to show the equiavlence of the two saddle point problems using the techniques of the Bounded Constraint Algorithm. The discussion is complete only when we prove the convergence of the Uzawa iterates (\ref{uzawa}). Using the Bounded Constraint Algorithm and the decoupling principle we show the following.
\begin{lemma}
	The Uzawa iterates can be shown to satisfy the bounds
	\begin{equation*}
		\|\lambda_1^{(n+1)} -\lambda_1^{(0)}\|_{\mathcal{H}}\le 2M_1\|f\|_{L^\infty}\|d^{(0)}\|_{\mathcal{H}} + r
	\end{equation*}
	where
	\begin{equation*} r=\max\Bigg\{\frac{\pi^2}{6}mCM,2\delta_{\text{HS}}\Bigg\}.
	\end{equation*}
\end{lemma}

\begin{proof}
	Set $n=i$ in Equation (\ref{uzawa})
	\begin{equation*}
		\lambda_1^{(i+1)} -\lambda_1^{(i)} = 2fd^{(i)} + \rho^{(i)}(B\textbf{u}^{(i)}-c),\quad 1\le i\le n.
	\end{equation*}
	Adding the $n$ equations we obtain
	\begin{align*}
		\begin{split}
			\lambda_1^{(n+1)} -\lambda_1^{(0)} &= 2f[d^{(n)}+\dots+ d^{(0)}]+\rho^{(n)}(B\textbf{u}^{(n)}-c)+\dots
			+\rho^{(0)}(B\textbf{u}^{(0)}-c)
		\end{split}\\
		\begin{split}
			&=2f[S(t)^nd^{(0)}+\dots+d^{(0)}]+\rho^{(n)}(B\textbf{u}^{(n)}-c)+\dots
			+\rho^{(0)}(B\textbf{u}^{(0)}-c).
		\end{split}
	\end{align*}
	Therefore,
	\begin{equation*}
		\lambda_1^{(n+1)} -\lambda_1^{(0)}=2f\Big[\sum_{i=0}^n S(t)^i\Big]d^{(0)}+\sum_{i=0}^n \rho^{(i)}(B\textbf{u}^{(i)}-c).
	\end{equation*}
	Hence,
	\begin{equation}
		\|\lambda_1^{(n+1)} -\lambda_1^{(0)}\|_{\mathcal{H}}\le 2\|f\|_{L^\infty}\Bigg|\Big[\sum_{i=0}^n S(t)^i\Big]d^{(0)}\Bigg|+\sum_{i=0}^n |\rho^{(i)}|\|B\textbf{u}^{(i)}-c\|_{\mathcal{H}}.
		\label{lam}
	\end{equation}
	Let us first consider the second sum in (\ref{lam}). Following the algorithm we note that $\|B\textbf{u}^{(n)}-c\|_{\mathcal{H}}\le C\epsilon_1^{(n)}$. Since the tolerance limit $\epsilon_1^{(n)}$ is tightened after every update, there exists a $N$ such that for $n>N$ we can choose a $M>0$ such that $\epsilon_1^{(n)}\le M/(n+1)^2$ as long as $\epsilon_1^{(n)}>2\delta_{\text{HS}}$. Thus
	\[
	\|B\textbf{u}^{(n)}-c\|_{\mathcal{H}}\le C\frac{M}{(n+1)^2}.
	\]
	Also as step (10) of the algorithm suggests the tuning parameter $\rho^{(n+1)}$ is assigned the value of the previous iteration $\rho^{(n)}$. This value is fixed as long as the steps (9)-(11) run. Let us denote this fixed value by $m$. Combining these discussions we obtain
	\begin{equation*}
		\sum_{i=0}^n |\rho^{(i)}|\|B\textbf{u}^{(i)}-c\|_{\mathcal{H}}\le mCM\sum_{i=0}^n\frac{1}{(i+1)^2},
	\end{equation*}
	which remains finite as $n$ becomes large. Now suppose it takes $n$ iterations where $\epsilon_1^{(n)}>2\delta_{\text{HS}}$. From the $(n+1)^{th}$ iteration when $\epsilon_1^{(n+j)}<2\delta_{\text{HS}}, j=1,2,\dots$ the upper bound becomes $2\delta_{\text{HS}}$. Combining we have 
	\begin{equation*}
		\sum_{i=0}^n |\rho^{(i)}|\|B\textbf{u}^{(i)}-c\|_{\mathcal{H}}\le r,
	\end{equation*}
	where
	\begin{equation*} r=\max\Bigg\{\frac{\pi^2}{6}mCM,2\delta_{\text{HS}}\Bigg\}.
	\end{equation*}
	which remains finite as $n\to \infty$. Since $S(t)$ is a contraction we observe that as $n\to\infty$ the series of operators in the first term of (\ref{lam}) converges to $(I-S(t))^{-1}$ which is a bounded operator. Thus we have $\|(I-S(t))^{-1}d^{(0)}\|\le M_1\|d^{(0)}\|$. Hence, in the limiting case, we have
	
	\begin{equation*}
		\|\lambda_1^{(n+1)} -\lambda_1^{(0)}\|_{\mathcal{H}}\le 2M_1\|f\|_{L^\infty}\|d^{(0)}\|_{\mathcal{H}} + r
	\end{equation*}
	which is finite. The shows the convergence of the multipliers $\lambda_1^{(n)}$.
\end{proof}

\section{Experiments and Results}
\label{sec:9}
A direct implementation of the single-phase continuity equation-based refinement requires a very large value of the regularization parameter along with a HS initialization for accelarating convergence and a pyramidal grid for a stable scheme. In this section, we show the results of our two-phase method compared to the continuity equation based method on different datasets.

\subsection{Experiments on PIV Dataset}
We first tested our algorithm on the oseen vortex pair (see Figure (\ref{f1})) and compared our results with the continuity equation model. The Oseen vortex pair is a synthetic PIV sequence of dimension $500\times 500$. The vorticies are placed centered at the positions $(166.7,250)$ and $(333.3,250)$. The circumferential velocity is given by $v_\theta=(\Gamma/2\pi r)[1-\exp(-r^2/r_0^2)]$ with the vortex strength $\Gamma=\pm 7000 \text{ pixels}^2/\text{s}$ and vortex core radius $r_0=15$ pixels. For more details see \cite{Liu2}.

\begin{figure}[H]%
	\centering
	\subfloat{{\includegraphics[width=5cm]{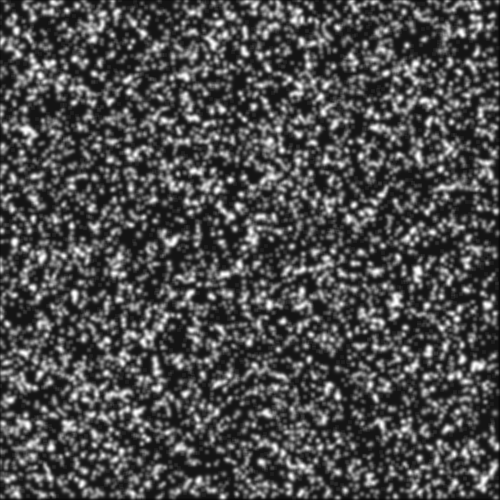}}}%
	\qquad
	\subfloat{{\includegraphics[width=5cm]{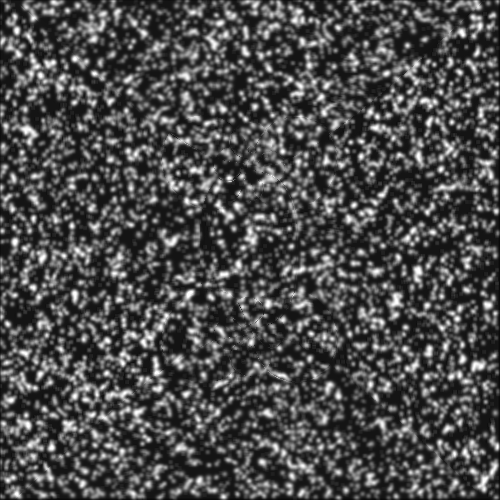}}}%
	\caption{Oseen vortex pair}%
	\label{f1}%
\end{figure}

\begin{figure}[H]%
	\hspace{-1.1cm}
	\subfloat[CEC with $L^2$ regularization]{{\includegraphics[width=9cm]{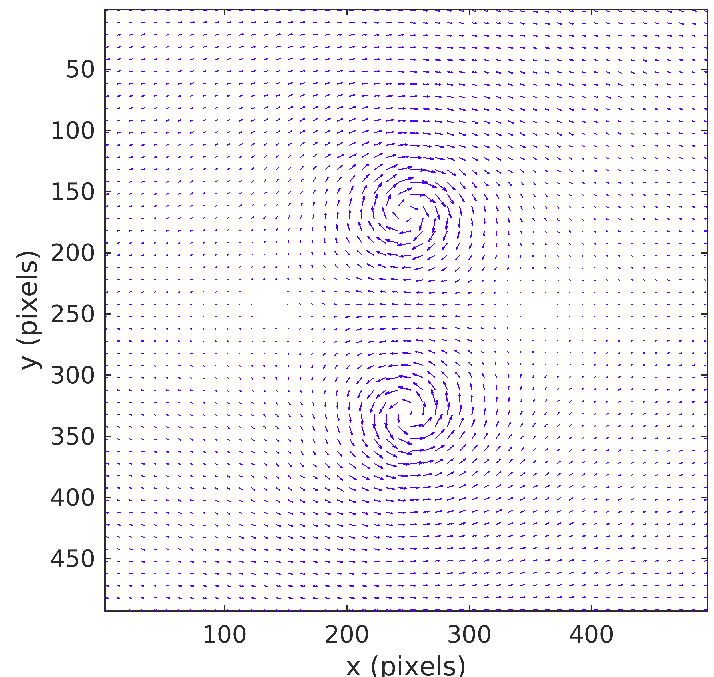}}}%
	\hspace{0.5cm}
	\subfloat[Our approach with $L^{2}$ regularizer]{{\includegraphics[width=9cm]{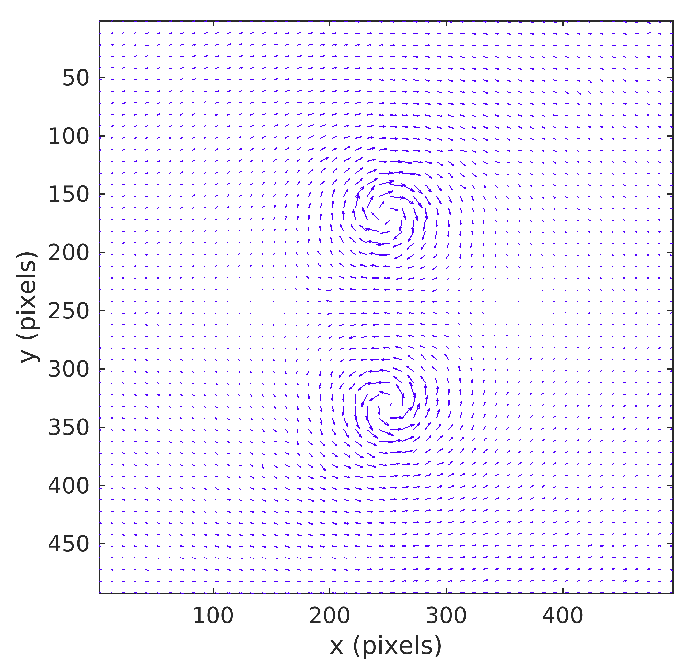}}}%
	\caption{Vorticity plot for the Oseen vortex pair}%
	\label{f2}%
\end{figure}

Figure (\ref{f2}) indicates that the vorticity plot obtained through our constraint-based refinement algorithm is very close to the continuity equation based model (CEC).

\begin{figure}[H]%
	\centering
	\subfloat{{\includegraphics[width=10cm]{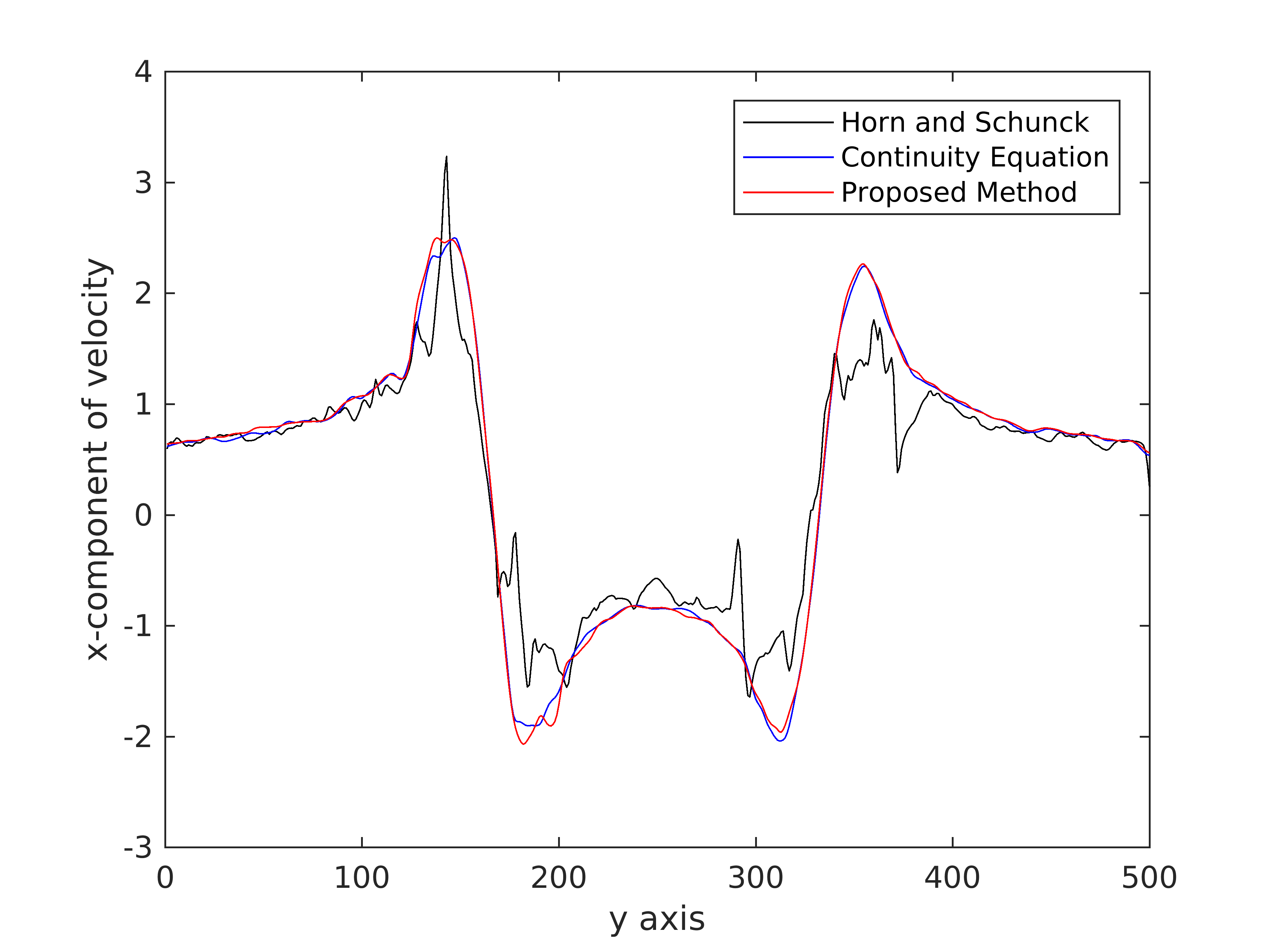}}}%
	\caption{Distribution of the $x$-component of the velocity extracted from the grid images for the oseen vortex pair}
	\label{f3}
\end{figure}

Performing a similar analysis as in \cite{Liu} we plot the distribution of the $x$-component of the velocity to obtain Figure (\ref{f3}). This plot compares the distribution of the $x$-component of the velocity extracted from the grid images for the HS model, CEC model and our constraint-based refinement model. The profiling clearly shows the closeness of our algorithm to the continuity equation based model. From the figure, it is also seen how the Horn and Schunck model underestimates the flow components, especially near the vortex cores.

For a better quantitative evaluation of our flow two-phase method, we consider additional PIV analytic flows \cite{fluid}. We show the results for the Poiseuille and Lamb-Oseen sequences, see \cite{fluid2} for the description of these datasets.

\begin{figure}[H]%
	\centering
	\subfloat[Ground-truth flow]{{\includegraphics[width=5.5cm]{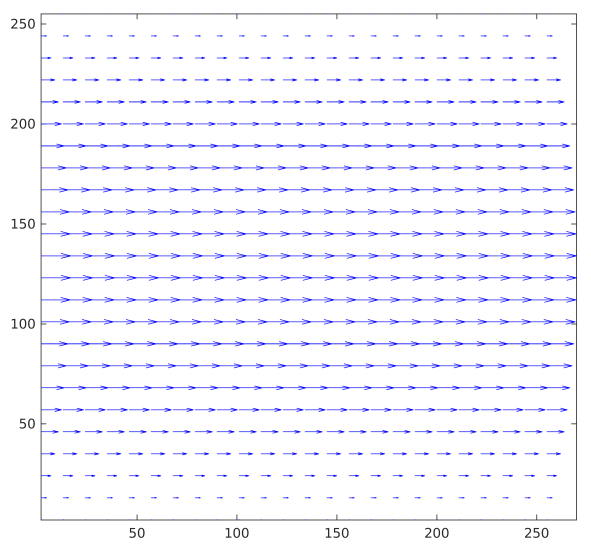}}}%
	\:
	\subfloat[CEC with $L^2$]{{\includegraphics[width=5.5cm]{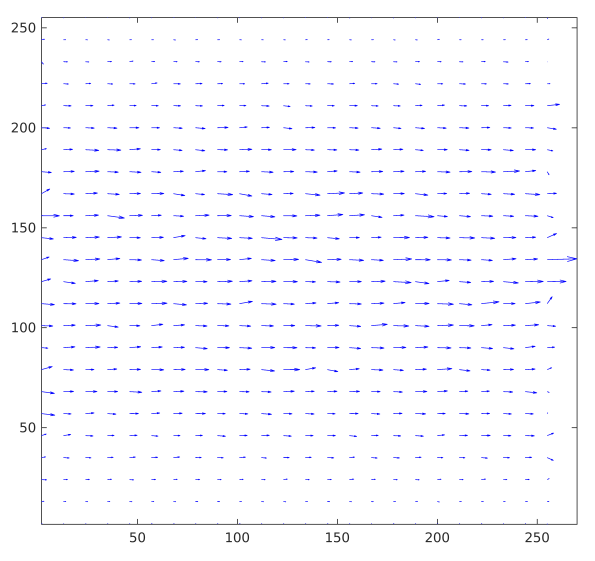}}}
	\:
	\subfloat[Our two-phase refinement]{{\includegraphics[width=5.5cm]{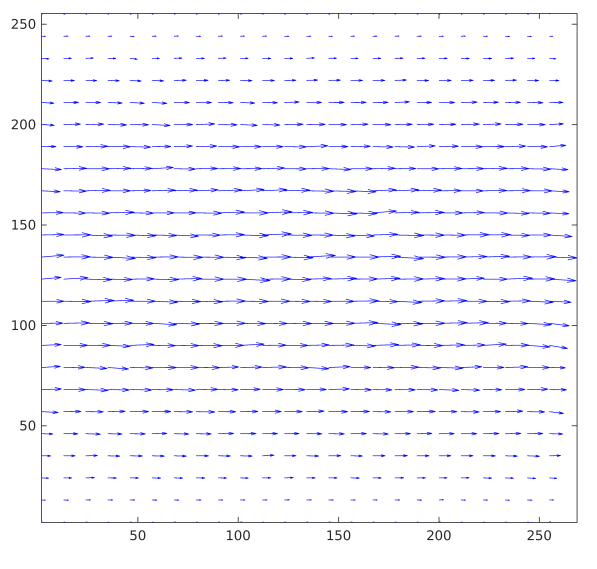}}}%
	\caption{Vorticity plot for the Poiseuille sequence}%
	\label{f11}%
\end{figure}

\begin{figure}[H]%
	\centering
	\subfloat[Ground-truth flow]{{\includegraphics[width=5.5cm]{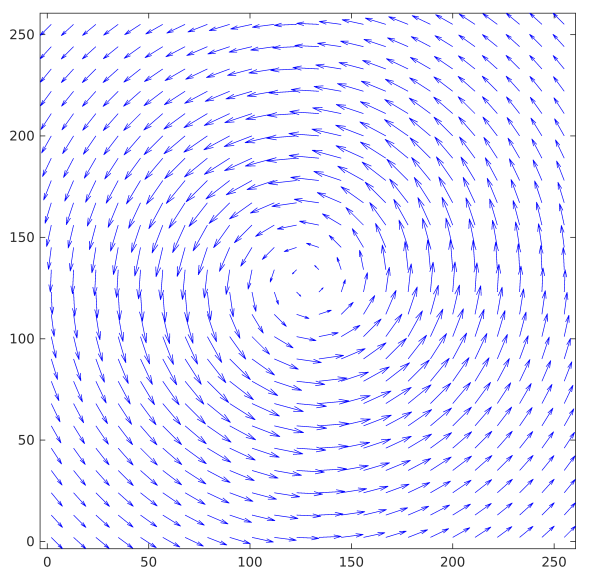}}}%
	\:
	\subfloat[CEC with $L^2$]{{\includegraphics[width=5.5cm]{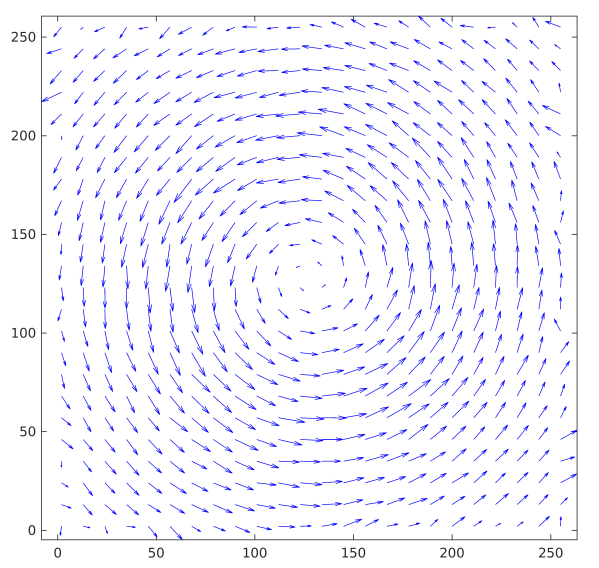}}}
	\:
	\subfloat[Our two-phase refinement]{{\includegraphics[width=5.5cm]{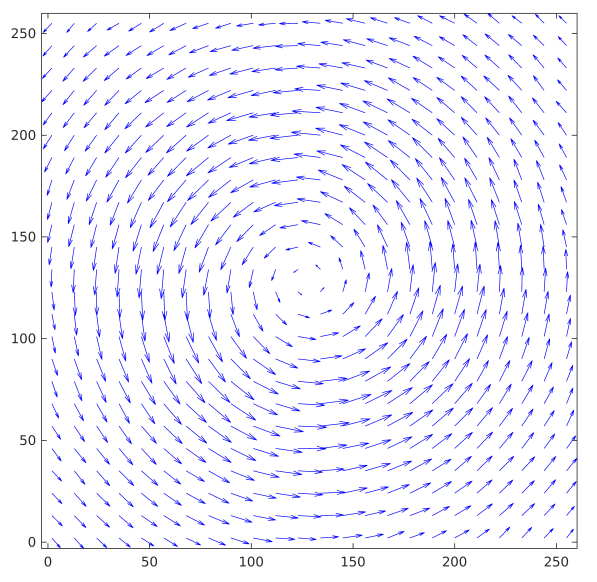}}}%
	\caption{Vorticity plot for the Lamb-Oseen sequence}%
	\label{f12}%
\end{figure}

\begin{table}[H]
	\renewcommand{\arraystretch}{1.5}
	\begin{adjustbox}{width=1.0\linewidth}
		\begin{tabular}{p{5.6cm}ccccc}
			\hline
			\multirow{2}{*}{Algorithm} & Poiseuille & Lamb-Oseen & Sink & Vortex & Potential flow \\
			& EPE  & EPE & EPE & EPE & EPE \\
			\hline
			\small CEC with $L^2$ regularization  & 0.181 & 0.847 & 0.031 & 0.265 & 1.260 \\
			\small Our Two-phase Refinement  &   \textbf{0.106} & \textbf{0.841} & \textbf{0.029} & \textbf{0.261} & \textbf{1.226} \\ \hline
			
		\end{tabular}
	\end{adjustbox}
	\caption{Comparison of the Average Angular Error (AAE) and End Point Error (EPE) for some PIV analytic flows.}
	\label{t2}
\end{table}

Figures \ref{f11} and \ref{f12} show the vorticity plot for the Poiseuille and the Lamb-Oseen sequences comapring both the methods with the ground-truth flow. Further, we report the End Point Error (EPE) for some of the PIV analytic flows in Table \ref{t2}. The EPE is computed as
\[
EPE = |\textbf{u}_e-\textbf{u}_c|=\sqrt{(u_1^e-u_1^c)^2 + (u_2^e-u_2^c)^2},
\] 
where $\textbf{u}_e= (u_1^e,u_2^e)$ is the exact optical flow, $\textbf{u}_c= (u_1^c,u_2^c)$ is the computed optical flow. From the table, it is conclusive that our two-phase method outperforms the CEC method.

\subsection{Experiments on Cloud Sequence}
In this sequence, the movement of the fluid exhibits both formation of a vortex as well as a movement of fluid parcels. The distribution of the strength of the vortices in the cloud sequence obeys a Gaussian distribution of mean 0 and standard deviation of 3000 (pixels)$^2/$s.

\begin{figure}[H]%
	\centering
	\subfloat{{\includegraphics[width=4cm]{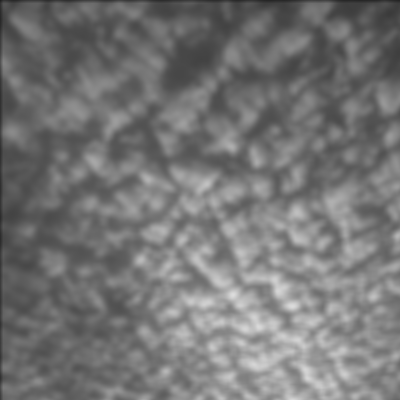}}}%
	\qquad
	\subfloat{{\includegraphics[width=4cm]{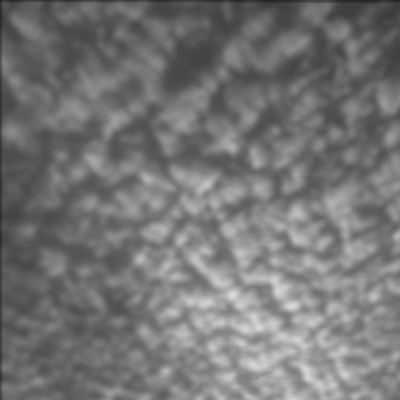}}}%
	\caption{Cloud sequence}%
	\label{f4}%
\end{figure}
Figure (\ref{f4}) shows the cloud sequence. The comparison of the velocity magnitude plots are shown below:
\begin{figure}[H]%
	\hspace{-1.1cm}
	\subfloat[CEC with $L^2$ regularization]{{\includegraphics[width=9cm]{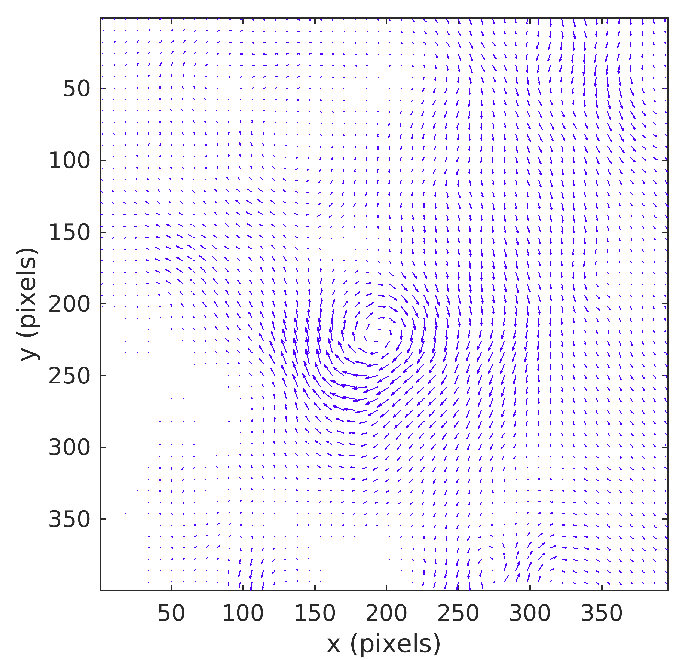}}}%
	\hspace{0.5cm}
	\subfloat[Our approach with $L^{2}$ regularizer]{{\includegraphics[width=9cm]{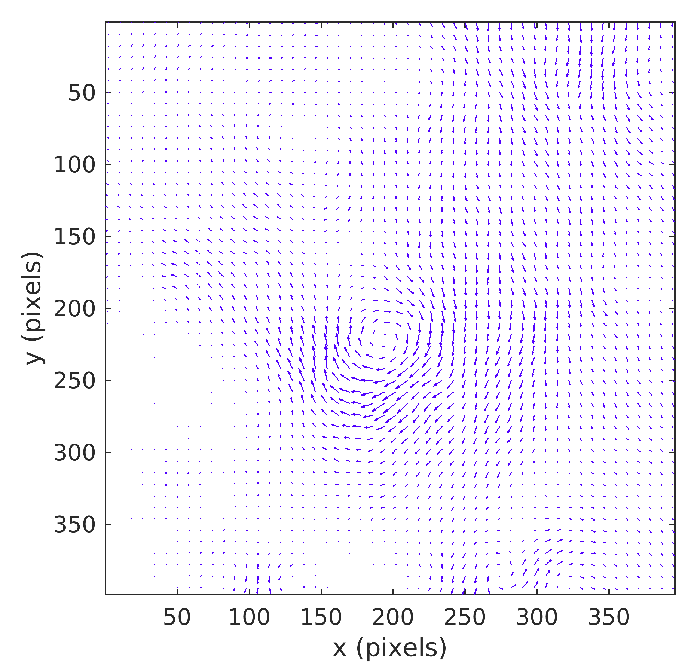}}}%
	\caption{Vorticity plot for the cloud sequence}%
	\label{f5}%
\end{figure}
As seen from Figure (\ref{f5}) the isotropic behaviour of the regularization is seen more on the continuity equation based implementation because of the denseness of the flow. Also by increasing the number of iterations we have observed that the effect of diffusion makes the vortices completely circular. The distribution of the $x$-component of the velocity for the cloud sequence is shown in Figure (\ref{f57}). The Horn and Schunck estimator tends to over estimate at the peaks.

\begin{figure}[H]%
	\centering
	\subfloat{{\includegraphics[width=11cm]{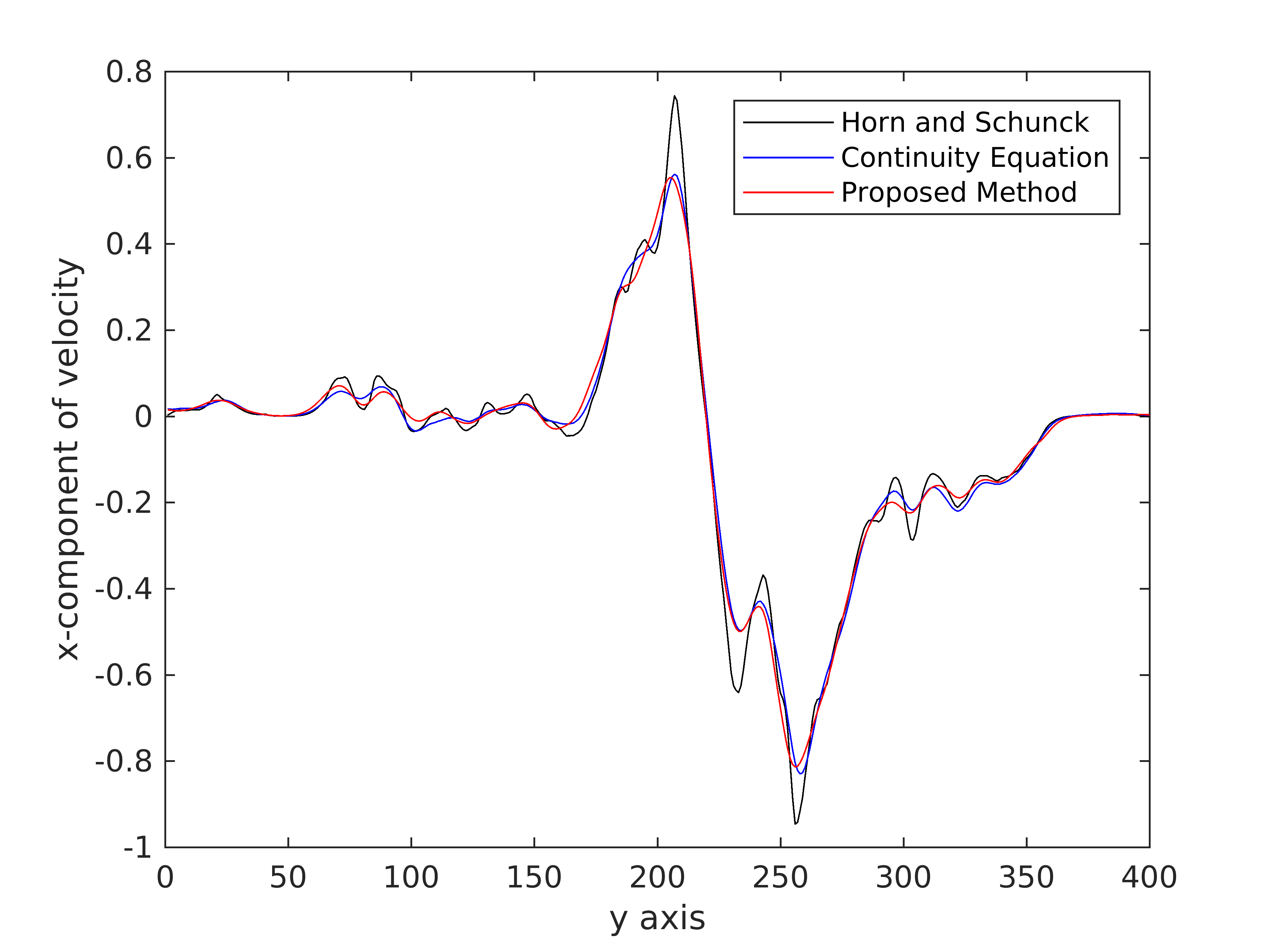}}}%
	\caption{Comparison between the distributions of the $x$-component of the velocity extracted from the grid images for the cloud sequence}
	\label{f57}
\end{figure}

\subsection{Experiments on Jupiter's White Oval Sequence}

Figure (\ref{f135}) shows Jupiter's white oval sequence. The white ovals seen in the images are distinct storms on Jupiter's atmosphere captured by NASA's Galileo spacecraft at a time-lapse of one hour, see \cite{Liu2}.

\begin{figure}[H]%
	\centering
	\subfloat{{\includegraphics[width=4cm]{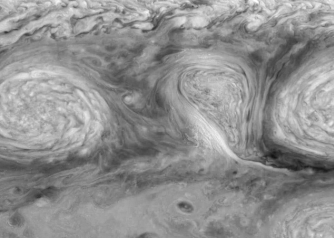}}}%
	\qquad
	\subfloat{{\includegraphics[width=4cm]{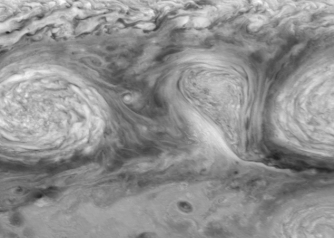}}}%
	\caption{Jupiter's white oval sequence}%
	\label{f135}%
\end{figure}

\subsubsection{Effect of Illumination Changes on Optical Flow}
Due to the time difference between adjacent frames, it was observed that the sun's illumination influenced the subsequent frame considerably in a non-uniform way. To compensate for the illumination effects, it is necessary to account for the illumination variation before applying the optical flow method. In Liu's implementation, an illumination correction is employed by normalizing the intensities and performing a local intensity correction using Gaussian filters. The first plot in Figure (\ref{f15}) shows the results of their implementation of the CEC model. 
\begin{figure}[H]%
	\hspace{-1.2cm}
	\subfloat[CEC + $L^2$ with illumination correction]{{\includegraphics[width=9cm]{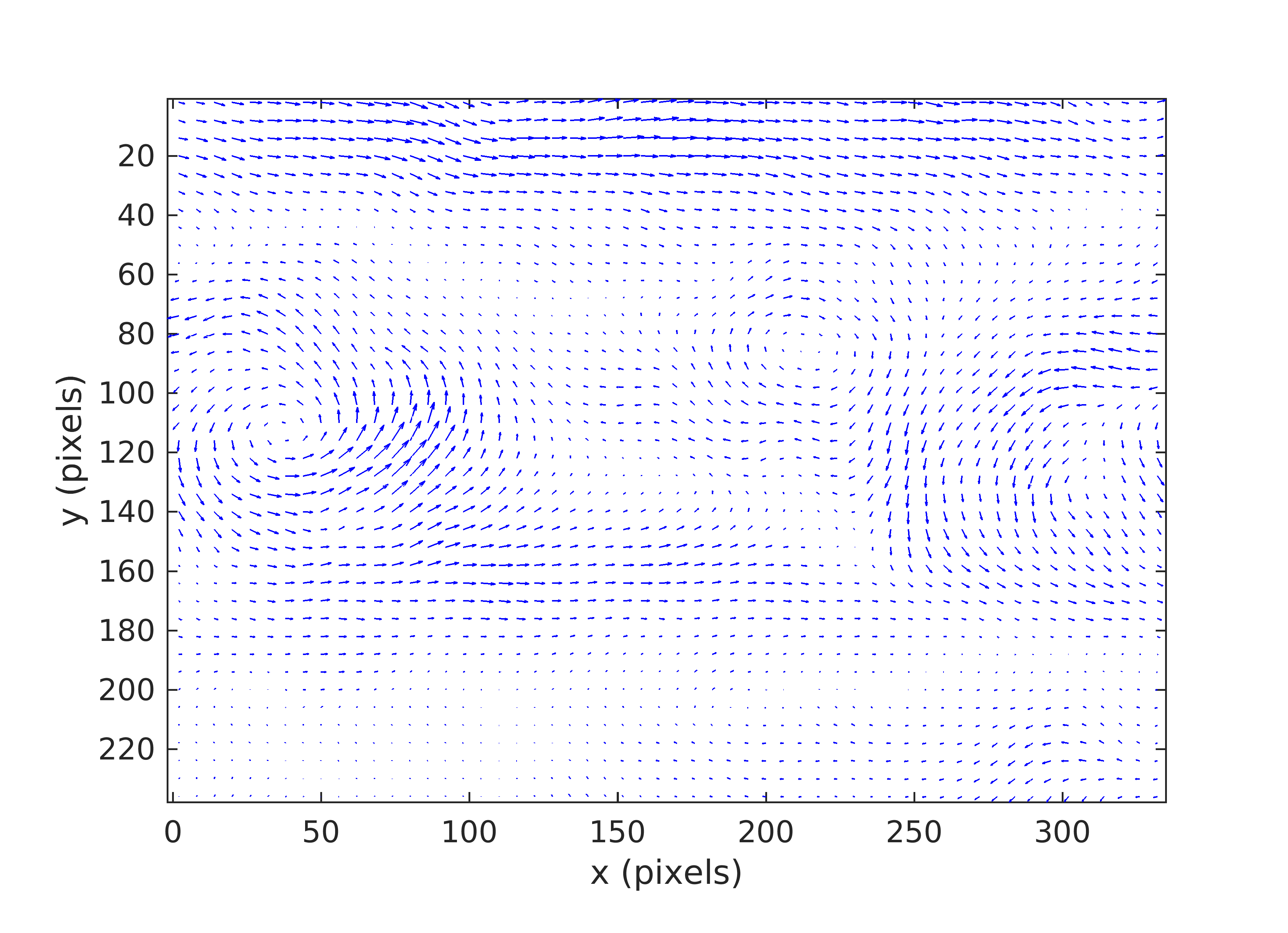}}}
	\subfloat[Our approach with illumination correction]{{\includegraphics[width=9cm]{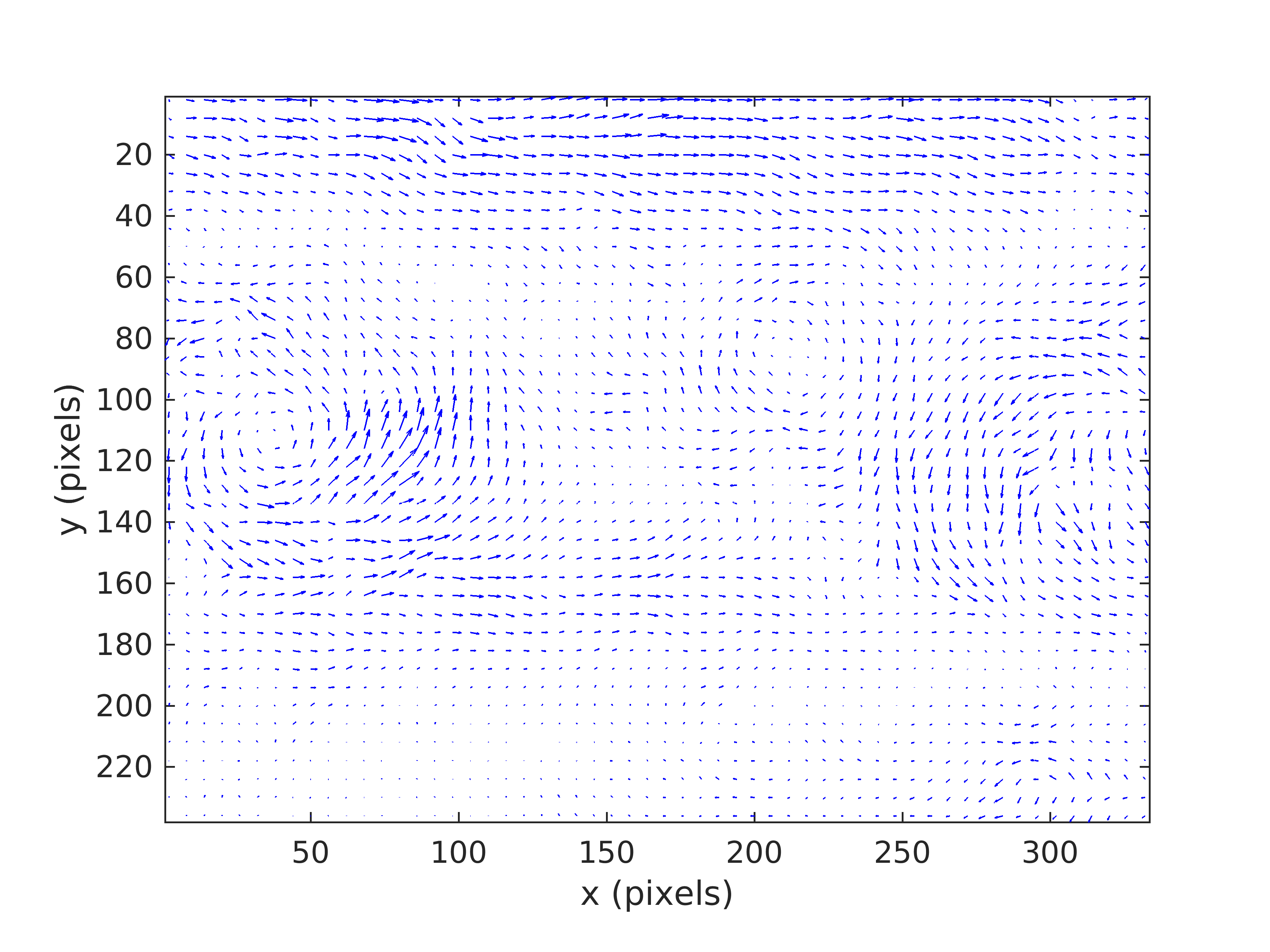}}}
	\caption{Velocity magnitude of the Jupiter's white oval sequence}%
	\label{f15}%
\end{figure}

The following comparison demonstrates the effect of illumination changes on the optical flow computation.
\begin{figure}[H]%
	\hspace{-1.1cm}
	\subfloat[Our approach without illumination correction]{{\includegraphics[width=9.3cm]{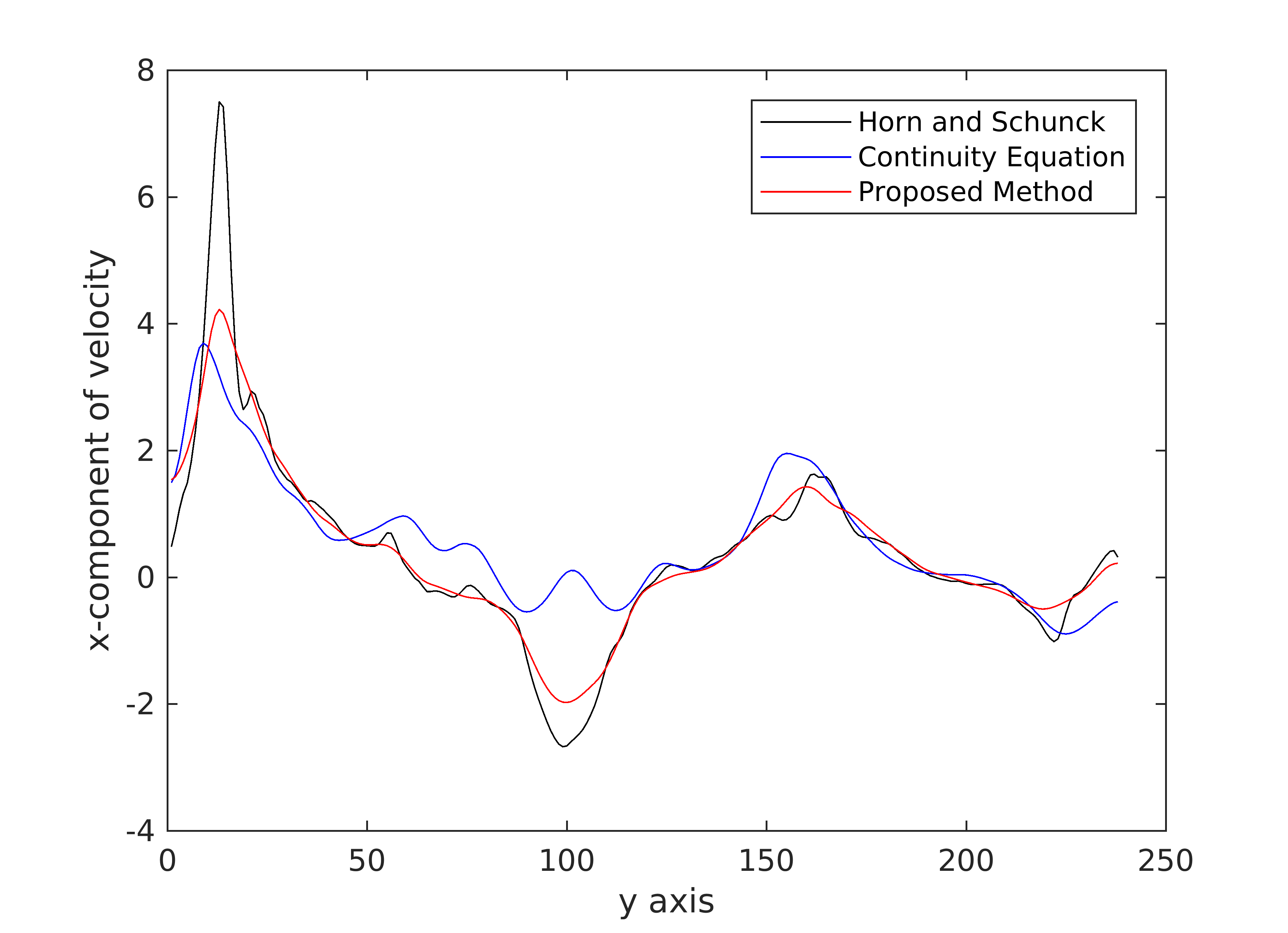}}}
	\subfloat[Our approach with illumination correction]{{\includegraphics[width=9cm]{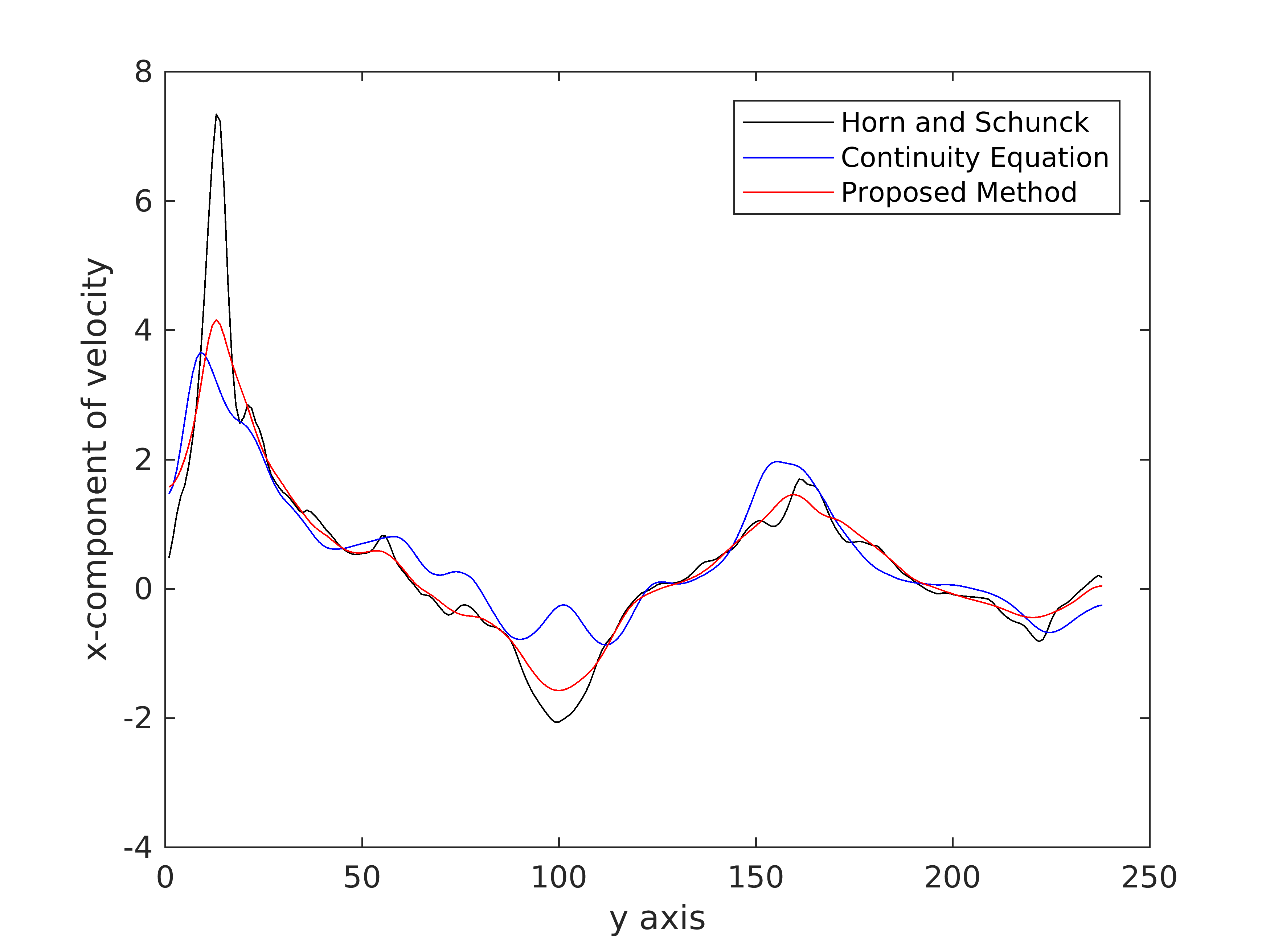}}}
	\caption{Effect of the illumination correction on the distribution of the $x$-component of the velocity for Jupiter's white oval sequence}%
	\label{f16}%
\end{figure}
As seen from Figure (\ref{f16}) there is a large deviation near the vortex region when illumination correction is not taken into account. The deviation is minimized to a great extent as can be seen from the second image. The reason for our results (even with illumination correction) not being very close to the illumination-corrected CEC-based flow is because of the direct dependency of the process on the image data.

\subsection{Demonstration of the Flow-Driven Refinement Process}

Rather than correcting the illumination changes by modifying the scheme we choose a flow-driven refinement process ($\phi(f)=1$) and perform a diffusion on the curl component. In order to achieve this, we consider the fourth case from Table (\ref{tab}), $\phi(f)=1$ and $\psi=(\nabla_H\cdot \textbf{u})^2$ where $\nabla_H=(-\partial_y,\partial_x)$ is the Hamiltonian  gradient. Introducing the symplectic gradient switches the roles of divergence and curl in the Equation (\ref{decsys}) and the analysis follows in the same lines. As mentioned earlier this particular choice captures the rotational aspects of the flow much better. 

\begin{figure}[H]%
	\hspace{-1.1cm}
	\subfloat{{\includegraphics[width=9.3cm]{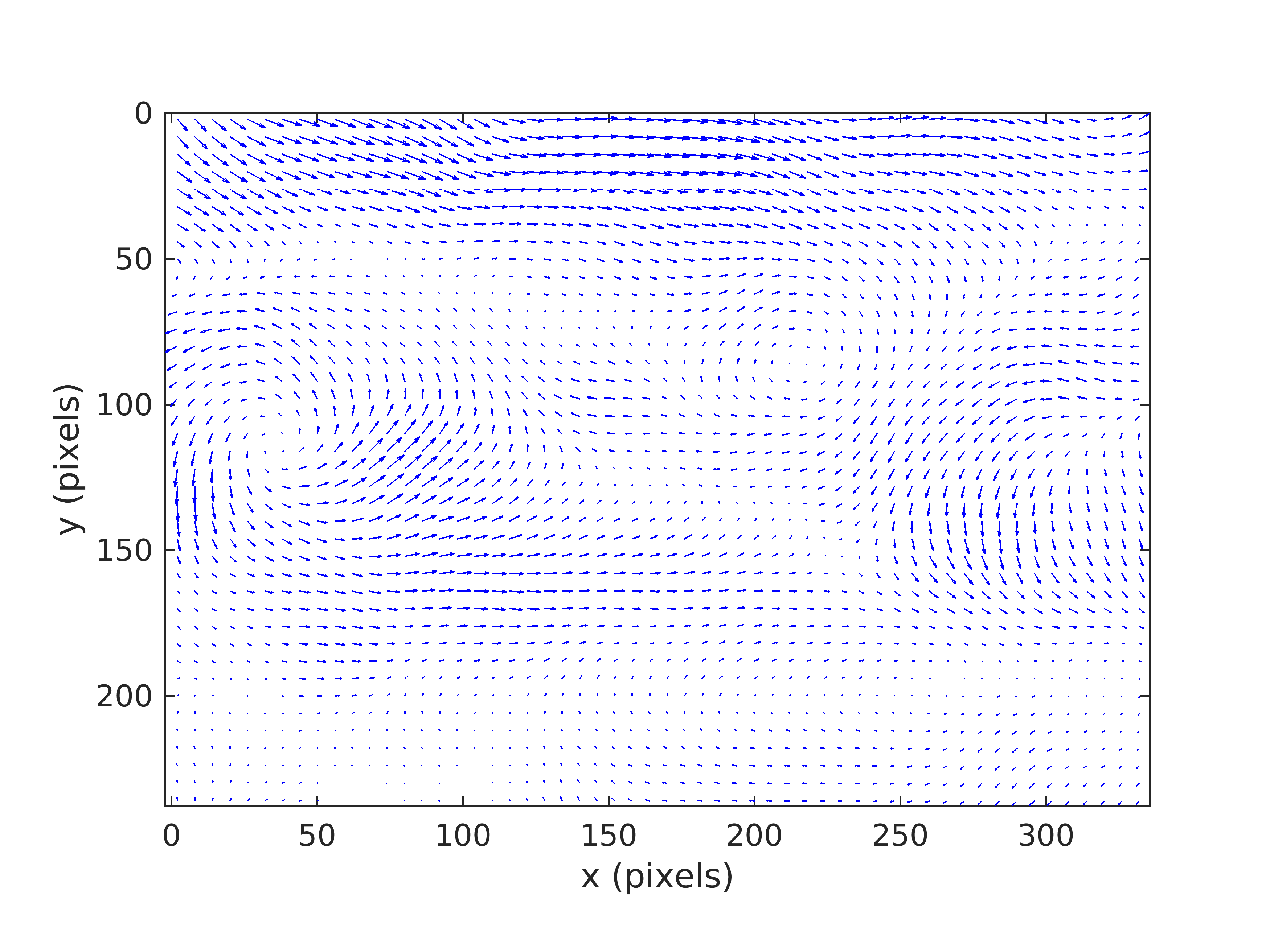}}}
	\subfloat{{\includegraphics[width=9cm]{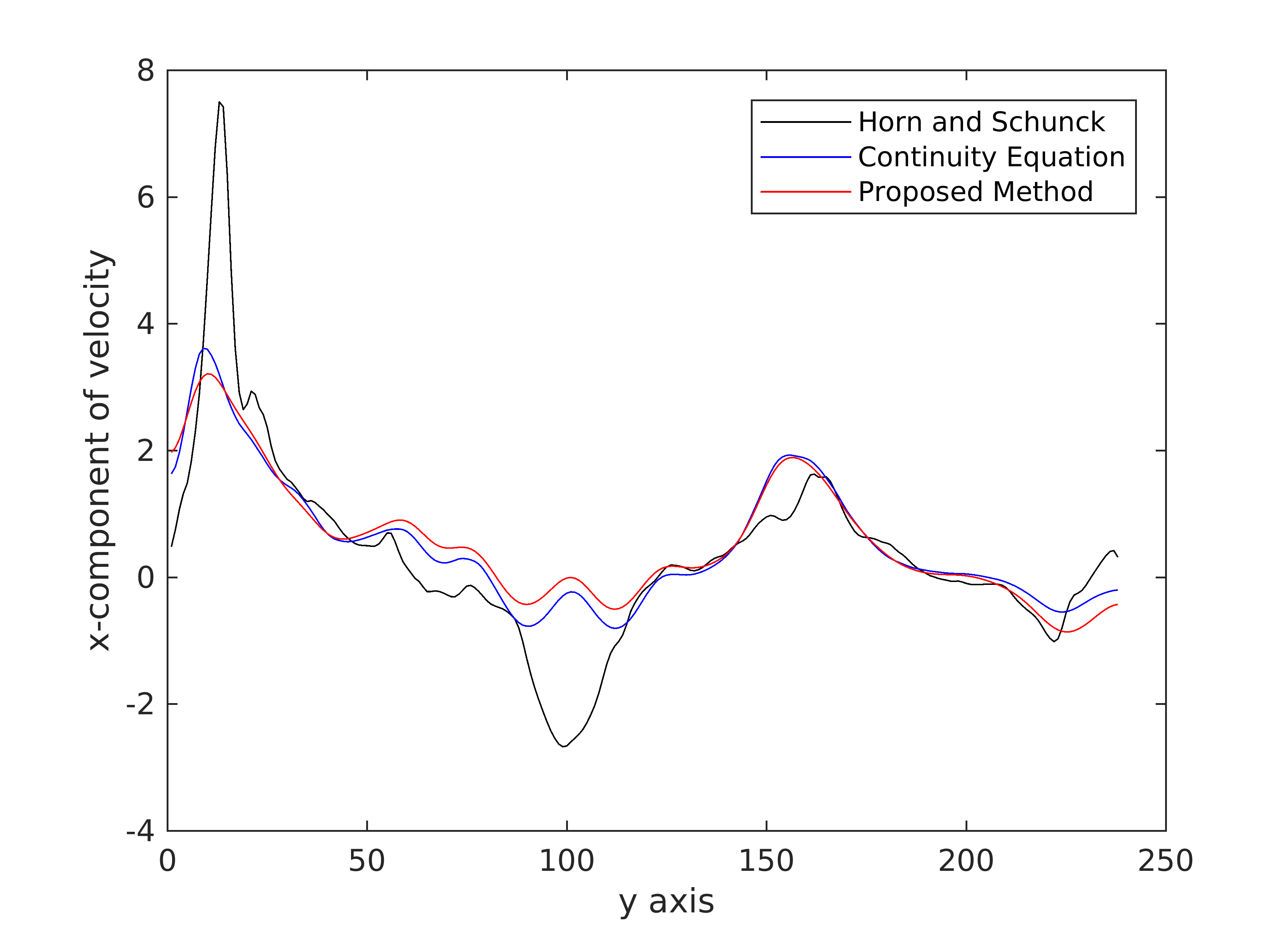}}}
	\caption{Plots of the Jupiter's white oval sequence with $\phi(f)=1,\psi(\mathbf{u})=(\nabla_{H}\cdot\mathbf{u})^{2}$ and without illumination correction.}%
	\label{f20}%
\end{figure}
Figure (\ref{f20}) gives the velocity magnitude plot obtained by our constraint-based refinement process $\phi(f)=1$ and $\psi(\mathbf{u})=(\nabla_{H}\cdot \mathbf{u})^{2}$ of Jupiter's white oval sequence along with the distribution of the $x$-component of the velocity. The ovals are clearly captured by our algorithm. From the distribution of the velocity plot, it is also clear that the flow-driven refinement process involving the curl outperforms the CEC-based flows without the additional asssumption of illumination correction on the image data.

\subsection{Discussion on the Choice of Parameters}
In the Liu-Shen implementation of CEC based model, the Lagrange multiplier in the HS-estimator is chosen to be 20 and in the Liu-Shen estimator, it is fixed at 2000. They observed that for a refined velocity field it does not significantly affect the velocity profile in a range of 1000-20,000 except the peak velocity near the vortex cores in this flow. For the image sequences, the best result was obtained for the values $\alpha=100$ and $\beta=0.01$. It was also observed experimentally that the numerical scheme converges when the ratio $\beta/\alpha$ is less than or equal to $10^{-4}$. 

\subsection*{Conclusion}
We have proposed a general framework for fluid motion estimation using a constraint-based refinement approach. We observed a surprising connection to the Cauchy-Riemann operator that diagonalizes the system leading to a diffusive phenomenon involving the divergence and the curl of the flow. For a particular choice of the additional constraint, we showed that our model closely approximates the continuity equation based model by a modified augmented Lagrangian approach. Additionally, we demonstrated that a flow-driven refinement process involving the curl of the flow outperforms the classical physics-based optical flow method without any additional assumptions on the image data.

\subsection*{Acknowledgements}
We express our deep sense of gratitude to Bhagawan Sri Sathya Sai Baba, Revered Founder Chancellor, SSSIHL. We would like to thank Dr. Shailesh Srivastava for his insights into obtaining the velocity plots.

\subsection*{Data Availability Statement}
The image sequences (data) used for this study were accessed from the repository \url{https://github.com/Tianshu-Liu/OpenOpticalFlow} available publicly as a supplementary material to \cite{Liu2} and also from \url{http://fluid.irisa.fr/data-eng.htm}.

\section*{Statements and Declarations}
\subsection*{Competing Interests}
The authors declare that they have no competing interests.

\subsection*{Funding}
The authors did not receive any financial grant during the preparation of this manuscript.

\end{document}